\pdfoutput=1

\documentclass[11pt]{article}

\usepackage{eacl2023}

\usepackage{xfrac}
\usepackage{amsmath}

\usepackage{hhline}
 
\usepackage{mathtools}
\mathtoolsset{showonlyrefs,showmanualtags}

\usepackage{times}
\usepackage{latexsym}
\usepackage{amsfonts,amsmath,amsthm}
\usepackage{graphicx}
\usepackage{subfig}
\usepackage{array}
\usepackage{ifthen}
\newcolumntype{C}[1]{>{\centering\arraybackslash}m{#1}}

\usepackage[T1]{fontenc}
\usepackage{bm}

\usepackage{relsize}

\definecolor{mygreenua}{HTML}{F1F5EB}
\definecolor{myredda}{HTML}{FFE6E6}

\newcommand{\uahelper}[1]{\colorbox{myredda}{\smaller$\uparrow$#1}}
\newcommand{\dahelper}[1]{\colorbox{mygreenua}{\smaller$\downarrow$#1}}

\newcommand{\uaghelper}[1]{\colorbox{mygreenua}{\smaller$\uparrow$#1}}
\newcommand{\dabhelper}[1]{\colorbox{myredda}{\smaller$\downarrow$#1}}

\newcommand{\ua}[1]{\ifthenelse{\equal{#1}{0.00} \or \equal{#1}{0.000}}{}{\uahelper{#1}}}
\newcommand{\da}[1]{\ifthenelse{\equal{#1}{0.00} \or \equal{#1}{0.000}}{}{\dahelper{#1}}}
\newcommand{\uag}[1]{\ifthenelse{\equal{#1}{0.00} \or \equal{#1}{0.000}}{}{\uaghelper{#1}}}
\newcommand{\dab}[1]{\ifthenelse{\equal{#1}{0.00} \or \equal{#1}{0.000}}{}{\dabhelper{#1}}}

\usepackage[utf8]{inputenc}

\usepackage{microtype}
\usepackage{mathtools}

\newcommand{\rv}[1]{\mathbf{#1}}
\newcommand{\myvec}[1]{\mathbf{#1}}
\newcommand{\mat}[1]{\bm{#1}}
\newcommand{\mymat}[1]{\bm{#1}}
\newcommand{\elementrv}[1]{\textnormal{#1}}

\newcommand{\myfunc}[1]{#1}

\newcommand{\shaycomment}[1]{\textcolor{blue}{#1}}

\newcommand{\shuncomment}[1]{\textcolor{orange}{#1}}

\newcommand{\ignore}[1]{}

\newcommand{\myhalf}{\sfrac{1}{2}}
\newtheorem{lem}{Lemma}

\title{Gold Doesn't Always Glitter: Spectral Removal of \\ Linear and Nonlinear Guarded Attribute Information}

\author{Shun Shao $\qquad$ Yftah Ziser $\qquad$ Shay B. Cohen \\
Institute for Language, Cognition and Computation \\
School of Informatics, University of Edinburgh \\
10 Crichton Street, Edinburgh, EH8 9AB \\
\texttt{s.shao-11@sms.ed.ac.uk} \quad
\texttt{yftah.ziser@ed.ac.uk} \\
\texttt{scohen@inf.ed.ac.uk}}

\begin{document}
\setlength{\fboxsep}{1pt}
\maketitle
\begin{abstract}
We describe a simple and effective method (Spectral Attribute removaL; SAL) to remove private or guarded information from neural representations. Our method uses matrix decomposition to project the input representations into directions with \emph{reduced} covariance with the guarded information rather than \emph{maximal} covariance as factorization methods normally use. We begin with linear information removal and proceed to generalize our algorithm to the case of nonlinear information removal using kernels. Our experiments demonstrate that our algorithm retains better main task performance after removing the guarded information compared to previous work. In addition, our experiments demonstrate that we need a relatively small amount of guarded attribute data to remove information about these attributes, which lowers the exposure to sensitive data and is more suitable for low-resource scenarios.\footnote{Code is available at \url{https://github.com/jasonshaoshun/SAL}.}

\end{abstract}


\section{Introduction}
Natural language processing (NLP) models currently play a critical role in  decision-supporting systems. Their predictions are often affected by undesirable biases encoded in real-world data they are trained on. Making sensitive predictions based on irrelevant input attributes such as gender, race, or religion (\textbf{protected} or \textbf{guarded} attributes) impacts user trust and the practical broad utility of NLP methods.

In recent years, representation learning approaches have become the mainstay of input encoding in NLP. While representation learning has yielded state-of-the-art results in many NLP tasks, controlling or inspecting the information encoded in these representations is hard. Thus, using rule-based methods to remove unwanted information from such representations is often not feasible. In the context of protected attributes, \newcite{bolukbasi2016man} showed that word embeddings trained on the Google News corpus encode gender stereotypes. Later, \newcite{manzini2019black} expanded this work and showed that word embeddings trained on the Reddit L2 corpus \cite{rabinovich2018native} encode race and religion biases.

\begin{figure}
    \centering
    \includegraphics[width=3 in]{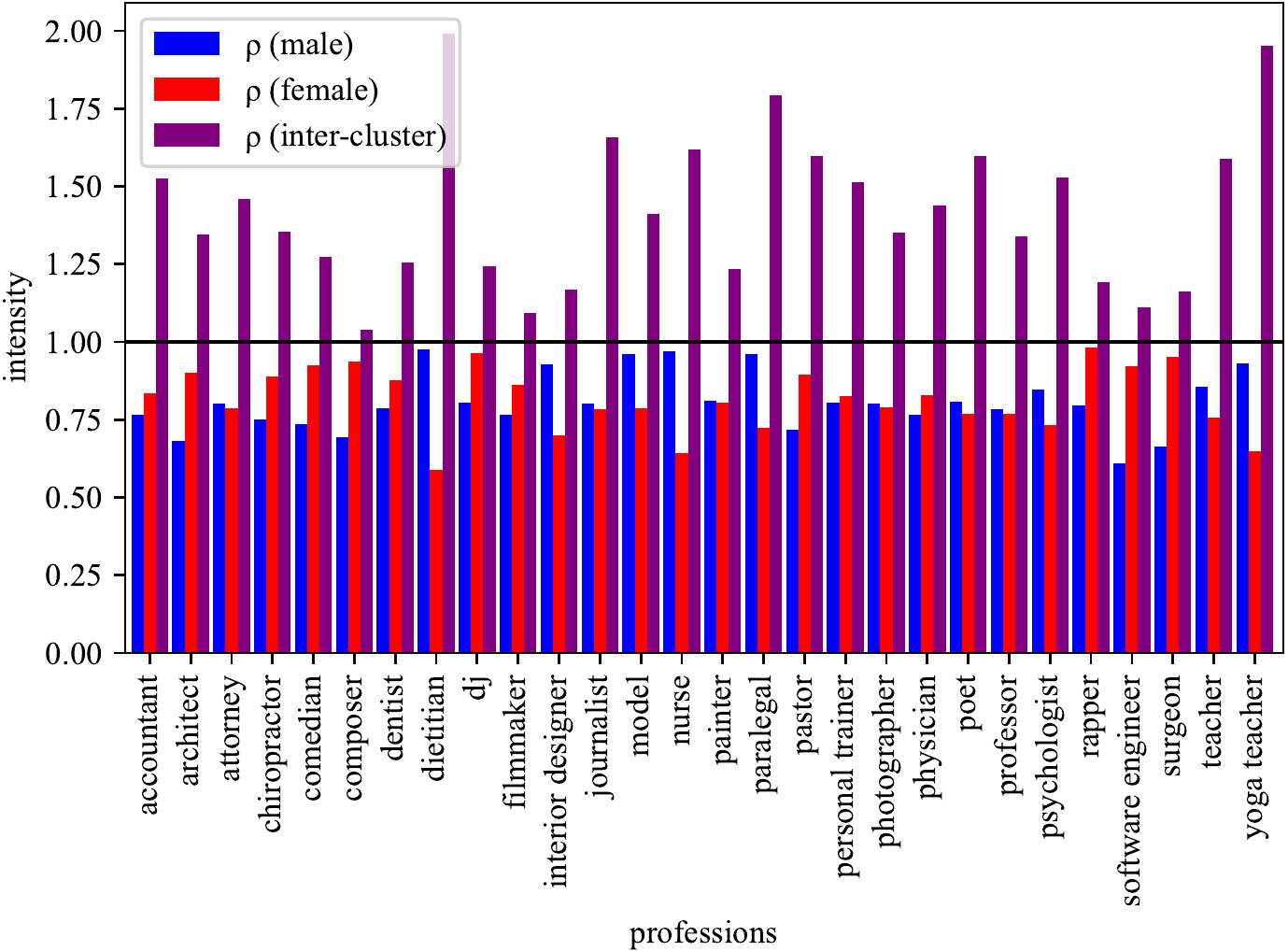}
    \caption{The ratio $\rho$ between the average t-SNE similarity of representations between two gender clusters $c_1,c_2$ ($\mathrm{sim}(c_1,c_2)$) for each profession: $\rho = \left(\textit{after SAL } \mathrm{sim}(c_1,c_2) \Big / \textit{before SAL }\mathrm{sim}(c_1,c_2)\right)$. Three values of $\rho$ are computed, intra-cluster: (1) $c_1=c_2=\text{male}$; (2) $c_1=c_2=\text{female}$; and inter-cluster: (3) $c_1=\text{male}, c_2=\text{female}$. The ratios in the inter-cluster case are smaller than $1$, and larger than $1$ for the intra-cluster case.}
    \label{fig:my_label}
\end{figure}

We propose a simple yet effective technique to remove protected attribute information from neural representations. Our method, dubbed \textbf{SAL} for Spectral Attribute removaL,
applies Singular Value Decomposition (SVD) on a covariance matrix between the input representation and the protected attributes and prunes highly co-varying directions. Figure~\ref{fig:my_label} demonstrates how professional biography text representations from labeled gender clusters (each biography is marked with the gender of its subject; \citealt{de2019bias}) for different professions  expand after the use of SAL, and become closer, implying a higher spread of each profession representations after SAL (\S\ref{section:biography}).


In addition, we overcome the \textbf{linear removal limitations} of SAL and previous work by using eigenvalue decomposition of \textbf{kernel matrices} to obtain projections into directions with reduced covariance in the kernel feature space. We refer to this method as \textbf{kSAL} (for kernel SAL). 
\ignore{
\begin{figure}[t]
\centering    
\subfloat{\includegraphics[width=0.16\textwidth]{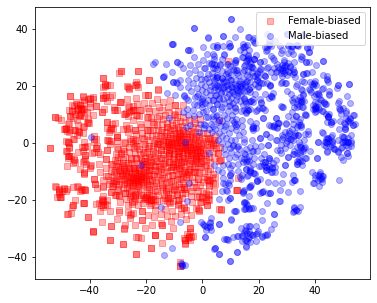}}
\subfloat{\includegraphics[width=0.16\textwidth]{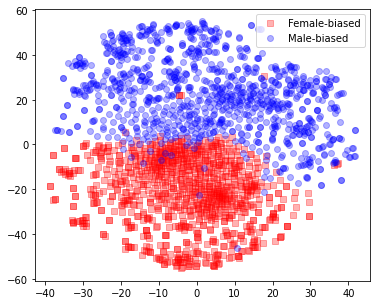}}
\subfloat{\includegraphics[width=0.16\textwidth]{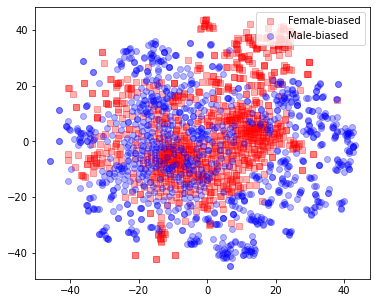}}
\caption{A t-SNE projection of GloVe word embeddings for $k=0,1,2$ (number of removed principal directions with SAL). The idea for the plot adapted from \newcite{ravfogel2020null}.}
\label{fig:pr}
\end{figure} 
}

SAL outperforms the recent method of \newcite{ravfogel2020null} aimed at solving the same problem and is able to remove guarded information much faster while retaining better performance for the main task. Further experiments demonstrate that our method performs well even when the available data for the protected attributes is limited.

\section{Problem Formulation and Notation}
\label{section:problem}

For an integer $n$ we denote by $[n]$ the set $\{1,\ldots,n\}$. For a vector $\myvec{v}$, we denote by $||\myvec{v}||_2$ its $\ell_2$ norm. Matrices and vectors are in boldface font (with uppercase or lowercase letters, respectively). Random variable vectors are also denoted by boldface uppercase letters. For a matrix $\mat{A}$, we denote by $\mat{A}_j$ its $j$th column (or by $\mat{A}_{i:j}$ the matrix with columns $\mat{A}_{k}$ for $k = i, \ldots, j$). Vectors are assumed to be column vectors.

In our problem formulation, we assume three random variables: $\rv{X} \in \mathbb{R}^d$, $\rv{Y} \in \mathbb{R}$ and $\rv{Z} \in \mathbb{R}^{d'}$.
Samples of $\rv{X}$ are the inputs for a classifier to predict corresponding samples of $\rv{Y}$. The random vector $\rv{Z}$ represents the guarded attributes. We want to maximize the ability to predict $\rv{Y}$ from $\rv{X}$, while minimizing the ability to predict $\rv{Z}$ from $\rv{X}$. Without loss of generality, we assume that the mean values of $\rv{X}$, $\rv{Y}$ and $\rv{Z}$ are $0$, and that $d' \le d$.\footnote{For example, $\rv{Z}$ may be a multi-class label such as gender represented as a short vector over $\{ -1, 1\}$ and $\rv{X}$ may be a complex input, which before removal of information about the guarded attribute $\rv{Z}$, can be used to predict $\rv{Z}$. An example of $\rv{X}$ would be an encoding of a post on a message board.}

We assume $n$ samples of $(\rv{X}, \rv{Y}, \rv{Z})$, denoted by $(\myvec{x}^{(i)}, \myvec{y}^{(i)}, \myvec{z}^{(i)})$ for $i \in [n]$.
These samples are used to train the classifier to predict the target values ($y$) from the inputs ($x$). These samples are also used to remove the information from the inputs based on the guarded attributes ($z$).

\section{Erasing Principal Directions}
\label{section:vanilla-svd}

We describe SAL in this section. We explain the use of SVD on cross-covariance matrices (\S\ref{section:vanilla-svd-1}) and describe the core algorithm in \S\ref{section:bb} and the connection to other algorithms in \S\ref{sec:conn}.

\subsection{SVD on Cross-covariance Matrix}
\label{section:vanilla-svd-1}

Let $\mat{A} = \mathbb{E}[\rv{X}\rv{Z}^{\top}]$, the matrix of cross-covariance between $\rv{X}$ and $\rv{Z}$. In that case, $\mat{A}_{ij} = \mathrm{Cov}(\elementrv{X}_i,\elementrv{Z}_j)$ for $i \in [d]$ and $j \in [d']$.

A simple observation is that for any two vectors $\myvec{a} \in \mathbb{R}^d, \myvec{b} \in \mathbb{R}^{d'}$, the following holds due to the linearity of expectation:
\begin{align}
\myvec{a} \mat{A} \myvec{b}^{\top} = \mathrm{Cov}(\myvec{a}^{\top} \rv{X}, \myvec{b}^{\top}\rv{Z}). \label{eq:A}
\end{align}

This motivates the use of the cross-covariance matrix to find the so-called principal directions: directions in which the projection of $\rv{X}$ and $\rv{Z}$ maximize their covariance, where the projections are represented as two matrices $\mat{U} \in \mathbb{R}^{d \times d}$ and $\mat{V} \in \mathbb{R}^{d' \times d'}$. Each column in these matrices plays the role of the vectors $\myvec{a}$ and $\myvec{b}$ in Eq.~\refeq{eq:A}. More specifically, we find $\mat{U}$ and $\mat{V}$ such that for any $i \in [d']$ it holds that:
\begin{align}
\mathrm{Cov}(\mat{U}_i^{\top} \rv{X}, \mat{V}_i^{\top} \rv{Z})  = \max_{(\myvec{a},\myvec{b}) \in \mathcal{O}_i} \mathrm{Cov}(\myvec{a}^{\top} \rv{X}, \myvec{b}^{\top}\rv{Z}) ,
\end{align}
\noindent where $\mathcal{O}_i$ is the set of pairs of vectors $(\myvec{a},\myvec{b})$ such that $||\myvec{a}||_2 = ||\myvec{b}||_2 = 1$, $\myvec{a}$ is orthogonal to $\mat{U}_1, \ldots, \mat{U}_{i-1}$ and similarly, $\myvec{b}$ is orthogonal to $\mat{V}_1, \ldots, \mat{V}_{i-1}$. 

It can be shown that such maximization can be done by applying the SVD on $\mat{A}$ such that
$\mat{A} = \mat{U}  \mat{\Sigma} \mat{V}^{\top}$,
where $\mat{U} \in \mathbb{R}^{d \times d}$, $\mat{\Sigma} \in \mathbb{R}^{d \times d'}$ and $\mat{V} \in \mathbb{R}^{d' \times d'}$.
In the case of SVD, $\mat{U}$ and $\mat{V}$ are orthonormal matrices, and $\mat{\Sigma}$ is a diagonal matrix with non-negative values on the diagonal. We let the vector of singular values on the diagonal of $\mat{\Sigma}$ be denoted by $\sigma_1, \ldots, \sigma_{d'}$. 

Once the orthogonal matrices in the form of $\mat{U}$ and $\mat{V}$ are found, one can truncate them (for example, use only a subset of the columns of $\mat{U}$, represented as the semi-orthonormal matrix $\hat{\mat{U}}$) to use, for example, $\hat{\mat{U}}^{\top}\rv{X}$, as a representation (linear projection) of $\rv{X}$ which co-varies the most with $\rv{Z}$.


We suggest that rather than using the largest singular value vectors in $\mat{U}$ to project $\rv{X}$, we should project $\rv{X}$ using the principal directions with the \textbf{smallest} singular values. This means we find a representative of $\rv{X}$ that co-varies the \emph{least} with $\rv{Z}$, essentially removing the information from $\rv{X}$ that is most related to $\rv{Z}$ and can be detected through covariance.

In addition, once such a projection matrix $\overline{\mymat{U}}$ is calculated, we can use the projection $\overline{\rv{X}} = \overline{\mymat{U}} \overline{\mymat{U}}^{\top} \rv{X}$ such that the value of $\mathbb{E}[||\rv{X} - \overline{\rv{X}}||_2]$ is minimized, while removing the information from $\rv{X}$.\footnote{This can be formalized using the min-max theorem of linear algebra, also referred to as the Courant–Fischer–Weyl min-max principle.} This allows us to potentially use the new projected values of the input random variable $\rv{X}$ \emph{without} changing a classifier that was originally trained on samples from $\rv{X}$, though as we see in \S\ref{section:exp}, using the projected input as-is without retraining the classifier may lead to performance issues with our method and other methods as well.



\ignore{
\paragraph{How is this related to Latent Semantic Analysis?} In LSA, we perform SVD on a word-document co-occurrence matrix, aiming to project either one type of vectors into a space that corresponds to the largest singular vectors.
}


\subsection{The SAL Algorithm}
\label{section:bb}
Our algorithm (SAL) follows the following procedure. First, the empirical cross-covariance matrix, estimating $\mathbb{E}[\rv{X}\rv{Z}^{\top}]$ is calculated:

\begin{equation}
\mat{\Omega} = \displaystyle\frac{1}{n} \sum_{i=1}^n \myvec{x}^{(i)}(\myvec{z}^{(i)})^{\top}.
\label{eq:omega}
\end{equation}

SVD is then performed on $\mat{\Omega}$ to obtain $(\mat{U}, \mat{\Sigma}, \mat{V})$. We choose an integer value $k$ and define $\overline{\mat{U}} = \mat{U}_{(k+1):d}$. The value of $k$ is bounded by the rank of $\mat{\Omega}$. The rank of $\mat{\Omega}$ is bounded from above by $d$ and $d'$, the dimensions of the vectors of $\myvec{X}$ and $\myvec{Z}$.

Then, the vectors $\myvec{x}^{(i)}$ are projected using either $\overline{\mat{U}}^{\top}$ or $\overline{\mat{U}}\overline{\mat{U}}^{\top}$. The latter projection attempts to project $\myvec{x}^{(i)}$ to the original dimensionality and space after removing the information. More specifically, $\overline{\mat{U}}\overline{\mat{U}}^{\top}$ is a projection matrix to the range of $\mymat{\Omega}$. 

The criterion we use to choose $k$ is based on the singular values in $\mat{\Sigma}$. More specifically, we choose a threshold $\alpha \ge 1$ and choose the minimal $k$ such that $\mat{\Sigma}_{11} / \mat{\Sigma}_{k+1,k+1} > \alpha$.


\subsection{Connection to CCA and PCA}
\label{app:connections}
\label{sec:conn}

We describe connections to other matrix factorization methods.

\paragraph{How is SAL related to Canonical Correlation Analysis?} The use of SVD on the cross-covariance matrix is very much related to the technique of Canonical Correlation Analysis (CCA), in which projections of $\rv{X}$ and $\rv{Z}$ are found such that they maximize the cross-correlation between these two random vectors. Rather than applying SVD on the cross-correlation matrix (CCA), we apply it on the cross-covariance matrix to preserve the $\rv{X}$ scale in our projection.

\paragraph{How is SAL related to Principal Component Analysis?} The use of SVD on the  cross-covariance matrix is  reminiscent of Principal Component Analysis (PCA), in  which eigenvalue decomposition is applied on $\mathbb{E}[\rv{X}\rv{X}^{\top}]$ to reduce the dimensionality of $\rv{X}$. However, PCA does not reduce  the  dimensionality of $\rv{X}$ while removing information present in the guarded r.v. $\rv{Z}$. Rather, it finds a $\rv{X}$ projection in which the covariance of a linear combination of $\rv{X}$ with itself is maximized.

In all three cases of CCA, PCA and in addition, LSA (Latent Semantic Analysis; \citealt{dumais2004latent}), SVD or eigenvalue decomposition is used with the aim of \emph{maximizing} the correlation or covariance between one or two random vectors. In our case, the SVD is used to \emph{minimize} the covariance between projections of $\rv{X}$ and $\rv{Z}$.

\ignore{
\subsection{Non-negativity of Singular Values}
\label{section:practical-basis}


In the previous sections, we assumed that $\sigma_i > 0$ for $i \in [d']$. It could be the case that $\sigma_i \approx 0$, either because the rank of $\mat{A}$ is smaller than $d'$ (i.e. there are less than $d'$ directions on which $\rv{X}$ and $\rv{Z}$ co-vary) or because of numerical issues with running SVD. This would lead to issues with identifiability --- \shaycomment{this is not the right place for this section, I am also not sure it is necessary}
}




\section{Kernel Extension to SAL}
To enrich the type of information that is detected as co-varying, it is possible to use two feature functions, $\phi \colon \mathbb{R}^d \rightarrow \mathbb{R}^m$ and $\psi \colon \mathbb{R}^{d'} \rightarrow \mathbb{R}^{m'}$, and apply the procedure in \S\ref{section:vanilla-svd} on $\mathbb{E}[\phi(\rv{X})(\psi(\rv{Z}))^{\top}]$. In that case, we can erase the information from $\phi(\rv{X})$ and treat it as the input for further classification. If the classifier is already learned, it would have to take input vectors of the form $\phi(\rv{X})$, otherwise, it can be re-trained with the erased inputs.

\subsection{The Kernel Trick}

The kernel trick refers to learning and prediction without explicitly representing $\phi(\myvec{x})$ or $\psi(\myvec{z})$. Rather than that, we assume two kernel functions, $\myfunc{K}_{\phi}(\myvec{x},\myvec{x'})$ and $\myfunc{K}_{\psi}(\myvec{z},\myvec{z}')$ that calculate similarities between two $x$s or between two $z$s. 

Every kernel that satisfies the necessary properties can be shown to be a dot product in some feature space. This means that for a given kernel function $\myfunc{K}_{\phi}(\myvec{x},\myvec{x'})$ it holds that

\begin{equation}
\myfunc{K}_{\phi}(\myvec{x},\myvec{x'}) = \langle \phi(x), \phi(x') \rangle, \label{eq:kernel}
\end{equation}

\noindent for some $\phi$ function and similarly for $\myfunc{K}_{\psi}(\myvec{z},\myvec{z}')$. Masking learning and prediction through a kernel function is often useful when the feature representations $\phi$ and $\psi$ are hard to explicitly compute, for example, because $m = \infty$ or $m' = \infty$ (such as the case with the Radial Basis Function, RBF, kernel).

We show next that the kernel trick can be used to generalize SAL to nonlinear information removal. 

\subsection{Removal with the Kernel Trick}
\label{section:kernel}

Rather than assuming a set of examples in the form mentioned in \S\ref{section:problem}, we assume we are given as input two kernel matrices of dimension $n \times n$:

\begin{align}
[\mat{K}_{\phi}]_{ij} = \myfunc{K}_{\phi}(\myvec{x}^{(i)}, \myvec{x}^{(j)}), \\
[\mat{K}_{\psi}]_{ij} = \myfunc{K}_{\psi}(\myvec{z}^{(i)}, \myvec{z}^{(j)}).
\end{align}

In addition, for the justification of our algorithm, we define the following two feature matrices based on the kernel feature functions:

\begin{align}
\forall i \in [m], j \in [n] & \, & [\mat{\Phi}]_{ij} = \phi(\myvec{x}^{(j)})_i, \\
\forall i \in [m'], j \in [n] &\, & [\mat{\Psi}]_{ij} = \psi(\myvec{x}^{(j)})_i.
\end{align}

Note that these two matrices are never calculated explicitly. Given the definition of the kernel as a dot product in the feature space (Eq.~\refeq{eq:kernel}), it can be shown that $\mat{K}_{\phi} = \mat{\Phi}^{\top} \mat{\Phi}$ and $\mat{K}_{\psi} = \mat{\Psi}^{\top}\mat{\Psi}$. In addition, we slightly change the empirical cross-covariance matrix $\mat{\Omega}$ definition in Eq.~\refeq{eq:omega} to:
$\mat{\Omega} =  \mat{\Phi}\mat{\Psi}^{\top}$. %
(This means we ignore the constant $1/n$ in the above definition of $\mat{\Omega}$, the constant that normalizes the matrix with respect to the number of examples. This does not change the nature of the following discussion, but it makes it simpler.) At this point, the question is how to perform SVD on $\mat{\Omega}$ without ever accessing directly the feature functions. This is where the spectral theory of matrices comes in handy.

More specifically, it is known that the left singular vectors of $\mat{\Omega}$ ($\mat{U}$) are the eigenvectors of $\mat{\Omega}\mat{\Omega}^{\top}$. In addition, the singular values of $\mat{\Omega}$ correspond to the square-root values of the eigenvalues of $\mat{\Omega}\mat{\Omega}^{\top}$.

In addition, we show in Appendix~\ref{appendix:A} why an eigenvector $\myvec{w}$ of $\mat{\Gamma} = \mat{K}_{\phi} \mat{K}_{\psi}$ can be transformed to an eigenvector of $\mat{\Omega}\mat{\Omega}^{\top}$ by multiplying $\myvec{w}$ on the left by $\mat{\Phi}$ and calculating $\mat{\Phi} \myvec{w}$.

With this fact in mind, we are now ready to find the left singular vectors of $\mat{\Omega}$ by finding the eigenvalues of $\mat{\Gamma}$, a matrix which is solely based on the kernel functions of $\myvec{x}$ and $\myvec{z}$.

Let $\myvec{w}_1, \ldots, \myvec{w}_k$ be eigenvectors of $\mat{\Gamma}$ and let $\myvec{w}'_1,\ldots,\myvec{w}'_k$ be
the orthonormalization of $\myvec{w}_i$, $i \in [k]$ based on the inner product $\langle \myvec{w}_i, \myvec{w}_j \rangle = \myvec{w}_i \mat{K}_{\phi} \myvec{w_j}^{\top}$. If we denote by $\mat{W}$ the matrix such that $\mat{W}_j = \myvec{w}'_j$ for $j \in [k]$, then $\Phi \mat{W} = \mat{U}$ where $\mat{U}$ is the left singular vector matrix of $\mat{\Omega}$. Then,

\begin{align}
    \mat{U}^{\top} \phi(\myvec{x}) & = (  \mat{W}^{\top} \mat{\Phi}^{\top}) \phi(\myvec{x}) = \mat{W}^{\top} \myfunc{\kappa}(\myvec{x}), \label{eq:kernel-proj}
\end{align}

\noindent where $\myfunc{\kappa}(\myvec{x})$ is a function that returns a vector of length $n$ such that $[\myfunc{\kappa}(\myvec{x})]_j = \myfunc{K}(\myvec{x}^{(j)}, \myvec{x})$. Eq.~\refeq{eq:kernel-proj} shows we can calculate the projection of $\phi(\myvec{x})$ while removing the information in $\psi(\myvec{z})$ by using the smallest eigenvalue eigenvectors of $\mat{\Gamma}$ and kernel calculations of each training example with $\myvec{x}$.

\subsection{Practical Kernel Removal}
\label{section:practical}

Using the kernel algorithm as above may lead to issues with tractability, as it possibly requires calculating the full eigenvector matrix of a large matrix (the product of two kernel matrices). We propose an alternative algorithm (kSAL) for the kernel case, which is more tractable.

For a fixed $0 \le k \le n$ (which does not need to be larger than the rank of either kernel matrices), we compute only the top $k$ eigenvectors of $\mat{\Gamma}$. We then compute an orthonormal basis for the null space of the matrix $(\mat{K}_{\phi,\myhalf{}} \mat{W}_{1:k})^{\top}$ where 
%
    $\mat{K}_{\phi,\myhalf{}} = \mat{U}_{\phi} \mat{\Sigma}_{\phi}^{1/2}\mat{V}_{\phi}^{\top}$,
%
with $(\mat{U}_{\phi}, \mat{\Sigma}_{\phi}, \mat{V}_{\phi})$ being the SVD of $\mat{K}_{\phi}$. Practically, this means we find a matrix $\mat{L}_{\phi} \in \mathbb{R}^{n \times (n-d)}$ such that $\mat{L}_{\phi}^{\top}\mat{L}_{\phi}=I$ and that $||(\mat{K}_{\phi,\myhalf{}} \mat{\Gamma})^{\top} \mat{L}_{\phi} ||_2 \approx 0$. The final data points $\hat{\myvec{x}}^{(j)}$ we use further down the pipeline correspond to the rows of $\mat{K}_{\phi,\myhalf{}} \mat{L}_{\phi} \in \mathbb{R}^{n \times (n-k)}$. If we are interested in using directly the reduced kernel matrix for the input vectors, we can use

\begin{equation}
    \hat{\mat{K}}_{\phi} = \mat{K}_{\phi,\myhalf{}} \mat{L}_{\phi} \mat{L}_{\phi}^{\top} \mat{K}_{\phi,\myhalf{}}^{\top}. \label{eq:khat}
\end{equation}

\paragraph{Time Complexity} Absorbing the kernel function computation as a constant, computing the kernel matrices is $\mathcal{O}(n^2)$ and their product $\mat{\Gamma}$ in $O(n^{\omega})$ for $\omega < 2.808$ using Strassen's algorithm, but can be done much more efficiently when $\mat{K}_{\psi}$ is sparse, as normally expected. Calculating the top $k$ eigenvectors of $\mat{\Gamma}$, has a cost of $\mathcal{O}(nk^2 + k^3)$ using, for example, the Arnoldi method.\footnote{For example, Matlab implements a variant of the Arnoldi method for its function \texttt{eigs}.} In \S\ref{section:kernel-exp2}, we report the clock running time for the kernel method.

Below, we experiment with RBF kernels (where $K_{\phi}(\myvec{x},\myvec{x}') = \exp(-\gamma||\myvec{x}-\myvec{x}'||_2^2)$; we use $\gamma = 0.1$) and polynomial kernel of degree 2 (where $K_{\phi}(\myvec{x},\myvec{x}') = (1 + \myvec{x}^{\top}\myvec{x}')^2$). The $\myvec{z}$ kernel remains linear (dot product).

\section{Experiments}
\label{section:exp}

In our experiments, our main comparison algorithm is the iterative null space projection (INLP) algorithm of \newcite{ravfogel2020null}, which aims at solving an equivalent problem to ours.
For the word embedding debiasing and fair classification (both setups), we follow the experimental settings of \newcite{ravfogel2020null}.\footnote{We use the authors' implementation for both the INLP method and the experimental settings: \url{https://github.com/shauli-ravfogel/nullspace_projection}.} 
SAL provides linear guarding, similarly to INLP, while kSAL also captures nonlinear regularities with respect to $\myvec{Z}$ (one-hot vector). We can provide such guarding for representations of state-of-the-art encoders (such as BERT), provided the representations are eventually fed into a classifier for prediction. The protected attributes we experiment with are \emph{gender} and \emph{race}. 

\ignore{
\shuncomment{The word embedding is always the first step in debiasing while the gender biases is one of most salient problems in the fields. However, the KNN does not quantify the changes in the embeddings and the semantic evaluation on word similarity and relatedness is still far from a real word application. Therefore, it is very practical to test the model on a fair classification dataset which contains both the main tasks label and biases label to balance the loss in the model performance and biases removal. We followed \citet{blodgett2016demographic} on race-sentiment-emoji (biases-main task-representation) fair classification,. However, the main tasks only has two labels and we find it is easy to get high accuracy on the sentiment prediction by our method. We used another more realistic data which has more complicated main tasks on profession prediction, this is the gender-profession-biography (biases-main task-representation). The dataset include 28 professions in total.}
}

\paragraph{Datasets}
For debiasing word embeddings (\S\ref{section:we}), we use 7,500 male and female associated words, 15K words overall. The dataset train/validation/test split sizes are (49\%/21\%/30\%). All the splits are balanced, i.e., containing an equal amount of male and female associated words. For the fair sentiment classification task (\S\ref{sec:fair-sentiment}), we use 10K training examples across all authors' ethnicity ratios (0.5, 0.6, 0.7, and 0.8). All training sets have an equal amount of positive and negative sentiment examples. The test set is balanced for both sentiment and authors' ethnicity labels. For the profession classification task (\S\ref{section:biography}), the data train/validation/test split sizes are (65\%/10\%/25\%), and all the splits combined contain 115K samples.

\subsection{Word Embedding Debiasing}
\label{section:we}

Word embeddings are often prone to encoding biases in various ways (see \S\ref{sec:related_work}). We evaluate our methods on gender bias removal from GloVe word embeddings. We use the 150,000 most common words and discard the rest. We sort the embeddings by their projection on the $\overrightarrow{\text{he}}$-$\overrightarrow{\text{she}}$ direction. Then we consider the top 7,500 word embeddings as male-associated words ($z=1$) and the bottom 7,500 as female-associated words ($z=-1$). 

\paragraph{Results with SAL}
\begin{table} [t]
\footnotesize
\centering
{
\begin{tabular}{|l|r|r|r|r|r cc}
\hhline{~----}
\multicolumn{1}{c|}{} &
\multicolumn{1}{|c|}{SL} & 
\multicolumn{1}{|c|}{WS-S} &
\multicolumn{1}{|c|}{WS-R} &
\multicolumn{1}{|c|}{Mturk} \\
\hline
Before   & 0.37 & 0.69 & 0.6 & 0.68 \\
After    & \uag{0.02} 0.39  & \uag{0.01}  0.7  &  0.6 & \uag{0.01} 0.69 \\
\hline
\end{tabular}
}
\caption{The semantic evaluation of word embeddings before and after removing gender bias.}
\label{tab:semantic_evaluation}
\end{table}

A linear classifier can perfectly predict the guarded gender attribute when trained on out-of-the-box GloVe embeddings. Removing the first direction ($k=1$) does not affect the accuracy demonstrated in Figure~\ref{fig:dir_removal}. For $k=2$, the performance drops to 50.2\%, almost a random guess.

We further perform intrinsic semantic tests to ensure the debiased embeddings remain useful. We use SimLex-999, WordSim353, and Mturk771 (similarity and relatedness datasets) to calculate the correlation between cosine similarities of the word embeddings to the human-annotated similarity score \cite{hill2015simlex, finkelstein2001placing, halawi2012large}.  We observed minor improvements for all tests when using debiased embeddings (Table~\ref{tab:semantic_evaluation}), suggesting that our method keeps the embeddings intact. We also report the three most similar words (nearest neighbors) for ten random words before and after SAL (see Appendix~\ref{appendix-b}). We observe almost no change between the two sets of embedding results.

SAL debiasing does not provide a nonlinear information removal. In Figure~\ref{fig:dir_removal} we plot the performance of nonlinear classifiers in the prediction of the \textbf{linearly-guarded} attribute (gender) as a function of the number of removed directions. We also provide linear classifier results for reference. We see that even after removing up to 30 principal directions, (linear) SAL is not sufficient for nonlinear classifiers -- the gender can still be predicted. This finding is also noted by \newcite{ravfogel2020null}, who did not offer a direct solution. This finding partially motivates our development of kSAL. 


\begin{figure}
\includegraphics[width=0.45\textwidth]{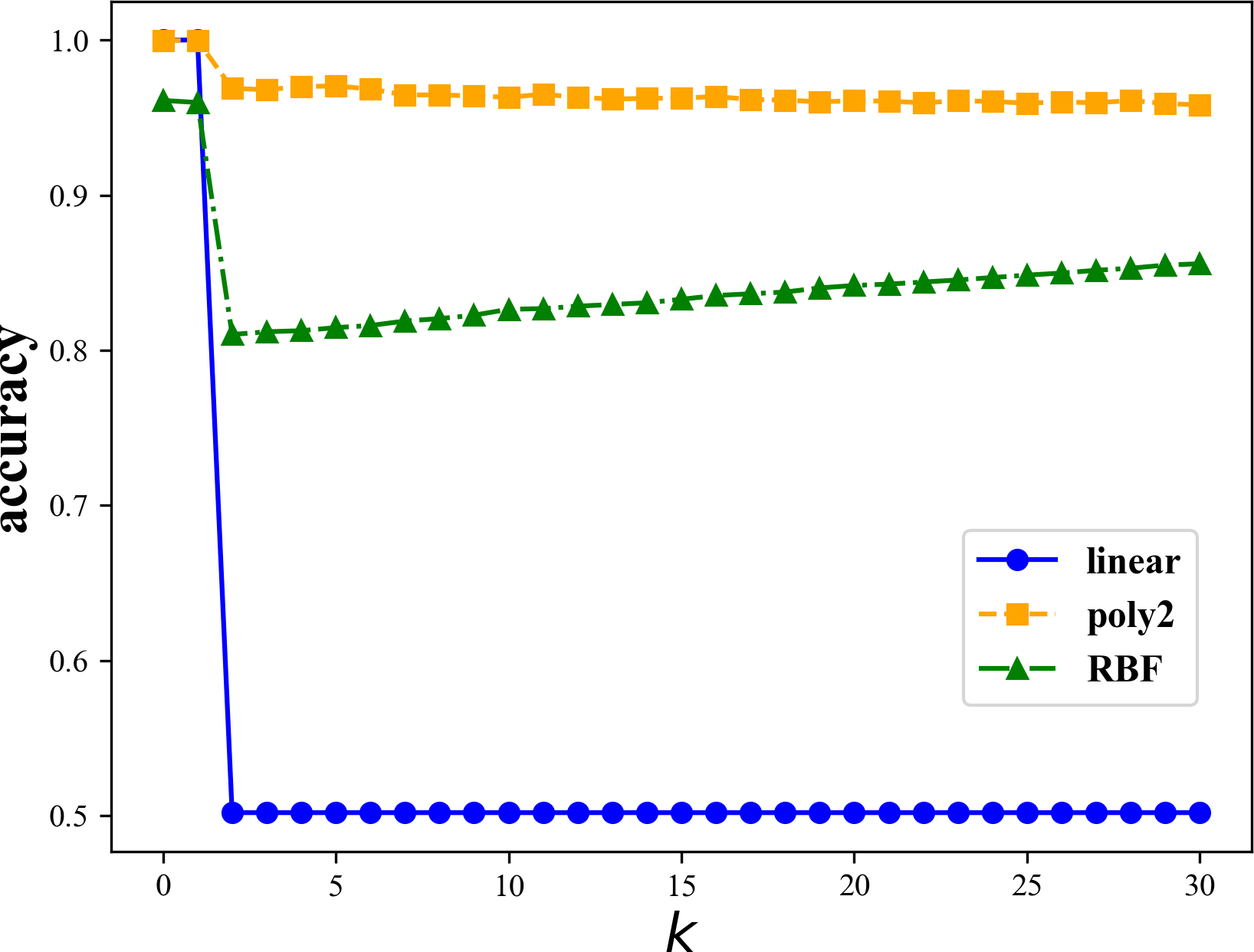}
\centering
\caption{A classifier accuracy for gender prediction as a function of the number of principal directions that are linearly removed. For the linear classifier, we use a linear SVM. For the nonlinear classifiers, we use SVM with the polynomial kernel and with the RBF kernel.}
\label{fig:dir_removal}
\end{figure}

\paragraph{Kernel Debiasing}
\label{section:kernel-exp}

All three kernels achieve high gender prediction accuracy when no information is removed ($k=0$), with accuracy of 100\%, 99.9\% and 95.7\% for the linear, polynomial, and RBF kernel, respectively. While the performance of the linear and polynomial kernels is not affected by removing one principal direction ($k=1$), the RBF kernel accuracy drops to 86.3\%. With $k=2$, performance drops to 50.2\%, 44.5\% and 50.2\% for the linear, polynomial, and RBF kernel, respectively, under \textbf{nonlinear kernel} removal. Compared to Figure~\ref{fig:dir_removal} with SAL, we see kSAL effectively removes nonlinear information. 


\paragraph{Deviations of Reduced Kernel from Original Kernel} To quantitatively test whether the embeddings retain their geometric form when removing gender information, we compare the standard deviation ($\rho$) of the values in $\mat{K}_{\phi}$ to the average deviation ($\gamma$) of values of $\mat{K}_{\phi}$ from the corresponding values in $\hat{\mat{K}}_{\phi}$ (Eq.~\refeq{eq:khat}).
When removing two principal directions, the largest approximation difference is seen in the linear kernel, with $\gamma / \rho = 0.64$. For the polynomial kernel, we observe $\gamma / \rho = 0.52$. For RBF, we have $\gamma / \rho = 0.16$. 

\subsection{Fair Classification}
\label{sec:fair-sentiment}
 To further evaluate our method on downstream tasks, we follow fair classification tests of social media text and other texts. 
 
\subsubsection{Fair Sentiment Analysis}
\begin{table*} [h!]
\scalebox{1.06}{
\small
\centering
{
\begin{tabular}{|c|r|r|r|r|r|r|r|r|c|c cc}
\hline
\multicolumn{1}{|c|}{} &
\multicolumn{4}{|c|}{Sentiment}&
\multicolumn{4}{|c|}{TPR-Gap}\\
\multicolumn{1}{|c|}{Rt} &
\multicolumn{1}{|c|}{Orig.} & \multicolumn{1}{|c|}{INLP} &
\multicolumn{1}{|c|}{SAL, $k=1$} & \multicolumn{1}{|c|}{SAL, $k=2$} &
\multicolumn{1}{|c|}{Orig.} & \multicolumn{1}{|c|}{INLP} &
\multicolumn{1}{|c|}{SAL, $k=1$} & \multicolumn{1}{|c|}{SAL, $k=2$}  \\
\hline
0.5    & 0.76 & 0.76  & 0.76 & 0.76 & 0.14 &\da{0.02} 0.12  &  0.14& \da{0.03} 0.11 \\
0.6    & 0.75 & 0.75 & 0.75 & 0.75 & 0.22 & \da{0.03} 0.19 & 0.22& \da{0.13} 0.09 \\
0.7    & 0.74 & 0.74 &  0.74 & 0.74 & 0.31 & \da{0.05} 0.26 & 0.31&   \da{0.15} 0.11 \\
0.8    &  0.72 & \dab{0.2} 0.52 &   0.72 &   0.72 &  0.40 & \da{0.39} 0.01 &\da{0.04} 0.36& \da{0.22} 0.18 \\
\hline
\end{tabular}
}}
\caption{The sentiment analysis scores (we use accuracy, as the dataset is balanced) and TPR differences (lower is better) as a function of the ratio of tweets (Rt) written by black individuals and
conveying positive sentiment. Arrows with numbers indicate absolute increase/decrease from the baseline, and their background color indicates a difference with positive implications (green) or negative ones (red).}
\label{tab:sen_fair}
\end{table*}

\paragraph{Task and Data}
The first task is sentiment analysis for social network users' posts. We use the TwitterAAE dataset \cite{blodgett2016demographic}, which contains users' tweets ($\myvec{x}$), coupled with the users' ethnic affiliations ($\myvec{z}$), and a binary label for the sentiment the tweet conveys ($\myvec{y}$). The dataset splits the users into two groups, African American English (AAE) speakers and Standard American English (SAE) speakers. As users' privacy makes it hard to obtain ground truth labels for ethnic affiliation, the dataset uses the demographics of the neighborhoods the users live in as a proxy. 
Following \newcite{ravfogel2020null}, we use the encoder of \newcite{felbo2017using}, DeepMoji, to obtain the tweets representation. DeepMoji is suitable for our goal, as it has been shown to encode demographic information and, therefore, might lead to unfair classification \cite{elazar2018adversarial}.

We experiment with four different setups. The dataset consists of an equal amount of positive and negative sentiment examples for all of them. The datasets differ with respect to the guarded attribute ratio. A ratio of $p \in \{ 0.5, 0.6, 0.7, 0.8 \}$ means that $p$ of the positive class examples are composed of AAE speakers, and $p$ of the negative class examples are composed of SAE speakers.  We experiment with ratios of $0.5$, $0.6$, $0.7$ and $0.8$. The larger the ratio, the higher the classifier's tendency to make use of protected attributes to make its prediction. 

\paragraph{Evaluation Measures}
We report the accuracy of the methods on the sentiment analysis task. To measure fairness, we use the difference in true positive rate (TPR-gap) between individuals belonging to different guarded attributes groups \cite{hardt2016equality,ravfogel2020null}. The rationale behind the TPR gap is that for an equal opportunity, a positive outcome must be independent of the guarded attribute ($\myvec{z}$), conditional on ($\myvec{y}$) being an actual positive. See \newcite{hardt2016equality} for more details.

\ignore{
\begin{equation}
\begin{split}
P(\hat{Y}=1|Z=z,Y=1)=\\
P(\hat{Y}=1|Z=\hat{z},Y=1)
\end{split}
\end{equation}
Our goal is hence to minimize the TPR-gap:
\begin{equation}
\begin{split}
P(\hat{Y}=1|Z=z,Y=1)-\\
P(\hat{Y}=1|Z=\hat{z},Y=1)=\\
TPR-gap_{z,y}
\end{split}
\end{equation}
}

\paragraph{Results}
Table~\ref{tab:sen_fair} presents our results for the fair sentiment classification. For the first three ratios, $0.5$, $0.6$, and $0.7$, we can see that both SAL ($k = 1, 2$) and INLP  maintain most of the main-task performance. In debiasing (TPR-Gap), SAL with $k=2$ significantly outperforms INLP. As expected, removing two directions results in better debiasing than removing one, but it does not lead to a performance drop on the main task. 
While for the last ratio, $0.8$, INLP achieves the highest TPR-gap result, it comes at the cost of a sharp performance drop on the main task, resulting in a nearly random classifier. SAL ($k=1,2$) maintains most of the main-task performance, and for $k=2$, the TPR-gap is halved.
\subsubsection{Fair Profession Classification}
\label{section:biography}
\begin{table*} [t]
\small
\centering
{
\begin{tabular}{|c|c|c|c|c|c|c|c|c|c|c cc}
\hline
\multicolumn{1}{|c|}{} &
\multicolumn{4}{|c|}{Accuracy (profession)}&
\multicolumn{4}{|c|}{TPR-Gap (RMS)}\\
\multicolumn{1}{|c|}{Encoder} &
\multicolumn{1}{|c|}{Orig.} & \multicolumn{1}{|c|}{INLP} &
\multicolumn{1}{|c|}{SAL, $k=1$} & \multicolumn{1}{|c|}{SAL, $k=2$} &
\multicolumn{1}{|c|}{Orig.} & \multicolumn{1}{|c|}{INLP} &
\multicolumn{1}{|c|}{SAL, $k=1$} & \multicolumn{1}{|c|}{SAL, $k=2$}  \\
\hline
FastText    & 0.75 & \dab{0.05} 0.71 & \uag{0.01} 0.76 & \uag{0.01} 0.76 & 0.20 & \da{0.11} 0.09 & \da{0.02} 0.18&  \da{0.08} 0.12 \\
BERT    & 0.8 & \dab{0.11} 0.69 & \dab{0.02} 0.78 & \dab{0.02} 0.78 & 0.21 & \da{0.15} 0.06 & \da{0.04} 0.17&  \da{0.12} 0.09 \\

\hline
\end{tabular}
}
\caption{The profession classification on the biographies dataset results. We report accuracy and TPR-RMS. The number of classes is 28.}
\label{tab:biasbios}
\end{table*}
\paragraph{Task and Data}
The second task is profession classification. 
\newcite{de2019bias} attempt to quantify the bias in automatic hiring systems and show that even for a simple task, predicting a candidate's profession based on a self-provided short biography, significant gaps result from the writer's gender. This might influence the open positions an automatic system will recommend to a candidate, thus favoring candidates from one gender over the other. 
We hence follow the setup of \newcite{de2019bias}, who experiment with professions classification ($\myvec{y}$), from short biographies ($\myvec{x}$), and gender as a guarded attribute ($\myvec{z}$). 
We use a multiclass classifier to predict the profession, as there are 28 profession classes. We experiment with two types of text representations, FastText \cite{joulin2016fasttext}, based on bag of word embeddings (BWE) and BERT \cite{devlin2018bert} encodings.

\paragraph{Evaluation Measures}
We report accuracy for the profession classification. For bias level measurement, we use a generalization of TPR-gap for multi-class, suggested by \newcite{de2019bias}, calculating the root mean square (RMS) of the TPR with respect to all classes.

\newcite{de2019bias} also provided evidence for a strong correlation between TPR-gap and existing gender imbalances in occupations, which may lead to unfair classification.
\paragraph{Results}
Table~\ref{tab:biasbios} presents the profession classification results. Similar to the sentiment analysis task, SAL ($k=1,2$) maintains most of the main-task performance, and for $k=2$, the two-direction removal, the TPR-gap is lower. When comparing SAL ($k=2$) to INLP, we observe a clear trade-off between maintaining the main task performance (SAL, $k=2$) and low TPR-gap scores (INLP).

\subsection{Scarce Protected Attribute Labels}
For many real-world applications, obtaining large amounts of labeled data for protected attributes can be costly, labor-intensive, and in some cases, infeasible due to an ever-increasing number of privacy regulations. In this analysis, we stress-test our algorithm by simulating a scenario in which only a limited amount of samples from the main task are coupled with the desired protected attribute labels. For this purpose, we replicate the fair sentiment classification experiments, but this time, feeding only a fraction of the annotated data to our debiasing method. The experiment is identical in terms of the main task, i.e., we use 100K samples for training the sentiment classifier. We experiment with different fractions of the debiasing data, i.e., 5\%, of the sentiment training data containing labels about the protected attribute. We hence feed  5,000 samples for debiasing. The subsets for debiasing are chosen randomly. We repeat each experiment 10 times with different subsets. Table~\ref{tab:sen_fair_low} presents our results. Using a small fraction of the data for debiasing did not significantly affect SAL's ($k=1,2$) main-task performance. INLP, on the other hand, suffers from a sharp performance decrease, resulting in a near-random sentiment classifier. SAL's ($k=1,2$) ability to debias the data is slightly worse than in the complete dataset setting but the resulting representations are still significantly less biased than the original ones. INLP achieves low TPR gaps, but it is hard to determine if this  is due to an accurate bias removal or a result of corrupting the representations.
\begin{table*} [t]
\centering
\scalebox{1.06}{
\small
{
\begin{tabular}{|c|c|c|r|r|c|c|r|r|c|c cc}
\hline
\multicolumn{1}{|c|}{} &
\multicolumn{4}{|c|}{Sentiment}&
\multicolumn{4}{|c|}{TPR-Gap}\\
\multicolumn{1}{|c|}{Ratio} &
\multicolumn{1}{|c|}{Orig.} & \multicolumn{1}{|c|}{INLP} &
 \multicolumn{1}{|c|}{SAL, $k=1$} &
 \multicolumn{1}{|c|}{SAL, $k=2$} &
\multicolumn{1}{|c|}{Orig.} & \multicolumn{1}{|c|}{INLP} &
\multicolumn{1}{|c|}{SAL, $k=1$} &
 \multicolumn{1}{|c|}{SAL, $k=2$}  \\
\hline

0.5    & 0.76 & \dab{0.19} 0.57 & 0.76 & 0.76 & 0.14 & \da {0.12} 0.02 & \ua {0.27} 0.41 & \da{0.03} 0.11 \\
0.6    & 0.75 & \dab{0.16} 0.59 & 0.75 & 0.75 & 0.22 & \da{0.19} 0.03 & \ua {0.01} 0.23 & \da{0.13} 0.09\\
0.7    & 0.74 & \dab{0.17} 0.57 & \dab{0.01} 0.73 & \dab{0.01} 0.73 & 0.31 & \da{0.26} 0.05 & 0.31 & \da{0.19} 0.12 \\
0.8    & 0.72 &  \dab{0.15} 0.57 & \dab{0.01} 0.71 & \dab{0.01} 0.71 & 0.40 & \da{0.32} 0.08 & \da{0.08} 0.34 & \da{0.27} 0.17 \\
\hline
\end{tabular}
}}
\caption{The sentiment analysis experiments, 100K samples are use to train the sentiment classifier, but only 5K examples are used for learning to remove bias. The test set is identical to the one used in \S\ref{section:exp}}
\label{tab:sen_fair_low}
\end{table*}
\subsection{Kernel Experiments}
\label{section:kernel-exp2}

\begin{table}[hbt!]
\centering
{
\begin{tabular}{|c|r|r|r|r|c|c|c|c|c|c cc}
\hline
\multicolumn{5}{|c|}{Sentiment Analysis (DeepMoji)}\\
\hline
\multicolumn{1}{|c|}{} &
\multicolumn{2}{|c|}{Main}&
\multicolumn{2}{|c|}{TPR-Gap}\\
\multicolumn{1}{|c|}{$k$} &
 \multicolumn{1}{|c|}{poly2} &
  \multicolumn{1}{|c|}{rbf} &
 \multicolumn{1}{|c|}{poly2} &
 \multicolumn{1}{|c|}{rbf}   \\
\hline
0    & 0.75 & 0.76 & 0.14 & 0.15 \\
1    & 0.75 & 0.76 & 0.14 & 0.15\\
2   & 0.75 & 0.76 & \da{0.01} 0.13 & \da{0.03} 0.12 \\
\hhline{|=|=|=|=|=|}
\multicolumn{5}{|c|}{Profession Classification (BERT)}\\
\hline
\multicolumn{1}{|c|}{} &
\multicolumn{2}{|c|}{Main}&
\multicolumn{2}{|c|}{TPR-Gap (RMS)}\\
\multicolumn{1}{|c|}{$k$} &
 \multicolumn{1}{|c|}{poly2} &
  \multicolumn{1}{|c|}{rbf} &
 \multicolumn{1}{|c|}{poly2} &
 \multicolumn{1}{|c|}{rbf}   \\
\hline
0    & 0.77  &  0.61 & 0.33 & 0.23 \\
1    & 0.77 & \uag{0.07} 0.68 & \da{0.05} 0.28 & \ua{0.11} 0.34\\
2    & 0.77 &  \uag{0.07} 0.68 & \da{0.08} 0.25 & \ua{0.08} 0.31 \\
\hline
\end{tabular}
}
\caption{Kernel results with kSAL for sentiment (for a ratio of 0.5) and profession classification.}
\label{tab:exp_kernel}
\end{table}
Despite their flexibility in modeling rich feature functions, kernels have been documented to be computationally intensive. Lack of computational resources prevented us from 
using the full sentiment and bios datasets for our kernel experiments, and instead, we use $15,000$ training examples and $7,998$ test set examples (the full test set) for the sentiment dataset and $15,000$ training examples and $5,000$ examples for the profession dataset. For training on the acquired $15,000$ training examples, we used one Intel Xeon E5-2407 CPU, running at 2.2 GHz, for approximately five hours (for a time complexity analysis, see \S\ref{section:practical}).

Table \ref{tab:exp_kernel} shows that using only a small subset of the data, kSAL-poly2 reduces the TPR gaps while maintaining almost identical performance to the original model on both the sentiment analysis and profession classification tasks. 
For the sentiment analysis task, kSAL-RBF slightly improves the main task results while reducing the TPR-gap (RMS). For the RBF profession classification task, the results are unexpected, with main task performance increasing as we remove principal directions. This could be due to the pruning of the rich, infinite feature space RBF kernel represents (we also observe significant overfitting with RBF).\footnote{With INLP, RBF-kernel SVM also obtains low-accuracy results.}

\subsection{Perturbed Inputs}
While the transformation through $\overline{\mymat{U}}\overline{\mymat{U}}^{\top}$ maps $\myvec{x}$ back into the original vector space (as a projection), it often turns out that it removes information in such a way that the \emph{original} classifier (trained on data without removal) can no longer be used with the inputs after removal. This issue exists not only with our algorithm, but also with INLP, and indeed, like us, \newcite{ravfogel2020null} re-trained their classifier after they created the cleaned projected inputs.

Ideally, we would want to remove information without necessarily having to retrain a classifier for the main task, as this is costly and perhaps unattainable. 
To test the effect of such an approach, we interpolated $\overline{\mymat{U}}\overline{\mymat{U}}^{\top}$ with the identity matrix, to eventually project $\myvec{x}$ using $\lambda \overline{\mymat{U}}\overline{\mymat{U}}^{\top} + (1-\lambda)\mymat{I}$ for $\lambda \in \{ 0, 0.1, \ldots, 1.0 \}$. This approach weakens the impact of the removal projection and retains some of the information in $\myvec{x}$. While an adversary can attack this approach,\footnote{Consider that the matrix $\lambda \overline{\mymat{U}}\overline{\mymat{U}}^{\top} + (1-\lambda)\mymat{I}$ could be invertible for $\lambda < 1$.}  it can mitigate the effects of privacy violations in cases where the service or software used with the modified representations cannot be retrained, especially if the service providers have no malicious intent.

Figure~\ref{fig:per_input} describes an ablation experiment, ranging $\lambda$ as above on the bios dataset. We see that as we increase the intensity of the use of the SAL projection (increasing $\lambda$), the accuracy of both gender prediction and profession prediction decrease when training the original classifier on the non-projected inputs. While the behavior is similar for the gender accuracy for both INLP and our method, the decrease for the profession prediction is much sharper for $\lambda > 0.4$ with INLP.

\begin{figure}
\includegraphics[width=0.45\textwidth]{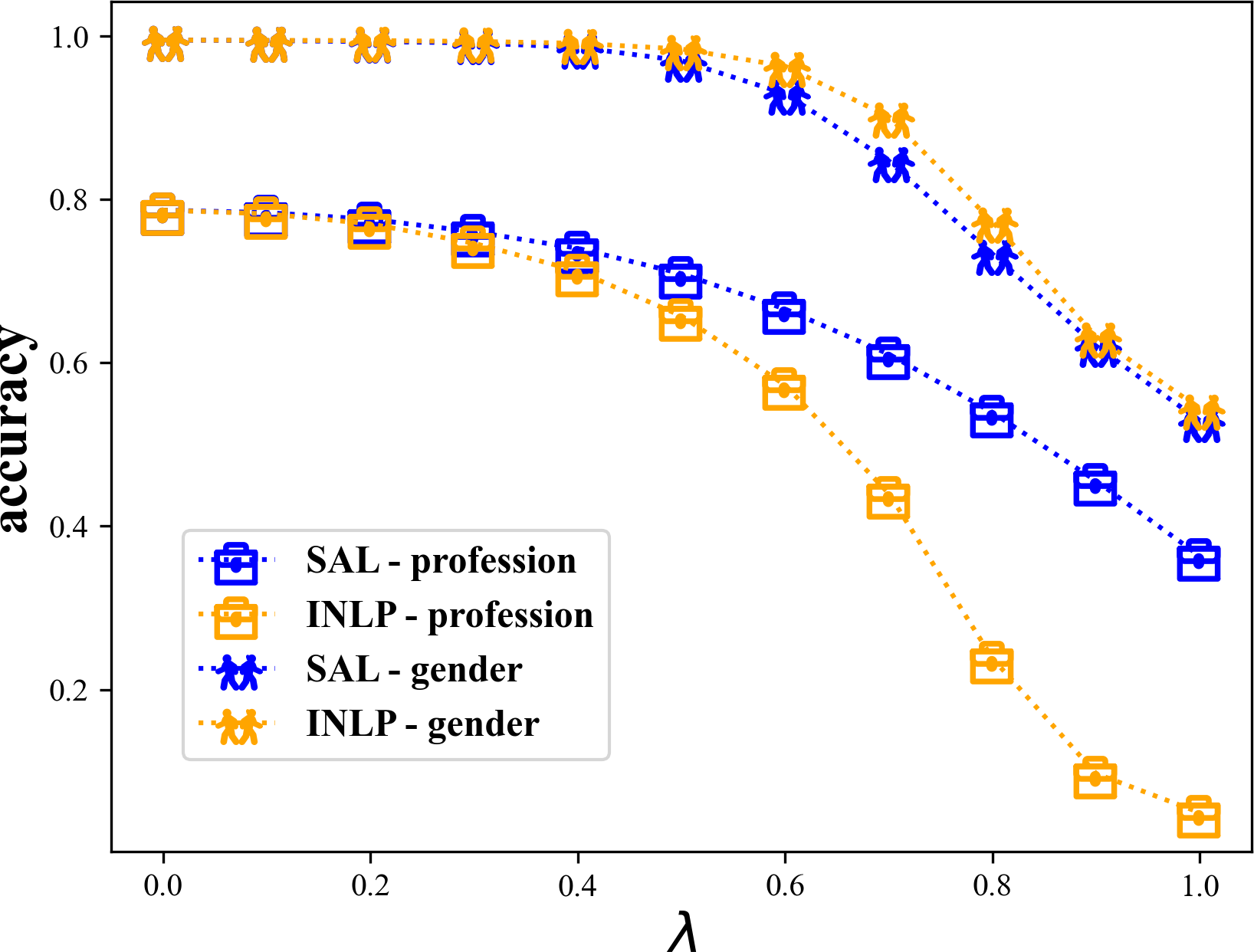}
\centering
\caption{Gender and profession classifications as a function of the interpolation coefficient $\lambda$.}
\label{fig:per_input}
\end{figure}

\begin{table} [t]
\scalebox{1.06}{
\small
\centering
{
\begin{tabular}{|c|l|l|l|l|}

\hline
\multicolumn{1}{|c|}{Task} &
\multicolumn{1}{|c|}{WED} & 
\multicolumn{1}{|c|}{FSC} &
\multicolumn{1}{|c|}{FPCF} & 
\multicolumn{1}{|c|}{FPCB} \\
\hline
SAL   & 0.03 sec & 0.37 sec & 0.16 sec & 0.35 sec  \\
INLP    & 50 sec & 100 min & 7 min & 35 min \\
\hline
\end{tabular}
}}
\caption{A run-time comparison between SAL and INLP. We used 2.20GHz Intel Xeon E5-2407 CPU for all of the experiments. WED, FSC, FSCF, and FPCB stand for word embedding debiasing, fair sentiment classification, and fair profession classification (with both FastText and BERT based representations).}
\label{tab:runtime}
\end{table}
\subsection{Runtime of SAL}
We measure the time it takes both methods to learn a projection matrix for a given training set. Once we have a projection matrix, debiasing the data is done by multiplying the data representation matrix by the learned projection matrix. Since matrix multiplication is a common practice for many research disciplines, and both methods use it, we do not benchmark it as well. Table \ref{tab:runtime} presents the run-time differences between SAL and INLP. For all of the experiments, SAL runtime is smaller by at least three orders of magnitude than INLP runtime.

\section{Related Work}
\label{sec:related_work}
In their influential work, \newcite{bolukbasi2016man} revealed that  word embeddings for many gender-neutral terms show a gender bias.
\newcite{zhao2018learning} presented a customized training scheme for word embeddings, which minimizes the negative distances between words in the two groups, e.g., male and female related words, for gender debiasing.  \newcite{gonen2019lipstick} demonstrated that bias remains deeply intertwined in word embeddings even after using the above methods.
For example, they showed several methods that can accurately predict the gender associated with gender-neutral words, even after applying the methods mentioned above. 
Similar to \newcite{ethayarajh2019understanding}, they concluded that removing a small number of intuitively selected gender directions cannot guarantee the elimination of bias. Motivated by this conclusion, \newcite{ravfogel2020null} presented iterative null space projection (INLP). This debiasing algorithm iteratively projects features into a space where a linear classifier cannot predict the guarded attribute. The debiased representations are \emph{linearly guarded}, i.e., they cannot guarantee bias removal beyond the linear level. Indeed, they show a simple nonlinear classifier can achieve high accuracy when predicting the guarded attribute.  Their approach is also related to that of \newcite{xu2017cleaning}. Previous work uses adversarial methods \cite{ganin2016domain} for information removal \cite{edwards2015censoring,li2018towards,coavoux2018privacy,elazar2018adversarial,barrett2019adversarial,han2021diverse} with the one by \newcite{ravfogel2022linear} being related to ours through the use of the mini-max theorem with the squared-error loss on the reconstruction of a matrix similar to our covariance matrix. In addition, methods based on similarity measures between neural representations \cite{colombo2022learning} were developed. To support the increasing interest in fair classification, \newcite{han2022fairlib} presented an open-source framework for standardizing the evaluation of debiasing methods. Finally, most relevant to this paper is an extension of SAL to the unaligned case, where protected attributes are not paired with input examples \cite{shun2023erasure}.

\section{Conclusions}
We presented a method for removing information from learned representations. We extended our method by using kernels, showing we can provide an effective nonlinear guarding. We also experimented with real-world low-resource situations, in which only a small  guarded attribute dataset is provided for information removal.

\section*{Limitations}

There are two main technical limitations to our work: (a) while the kernel removal is nonlinear, it still depends on a feature representation that captures a specific type of nonlinearities; (b) like other kernel methods, the kernel removal method is significantly slower than direct SVD removal in cases where the feature representations can be written out without the need of an implicit kernel. Future work may apply random projections to the kernel matrices to decompose them more efficiently.

A general limitation of current information removal methods is that they can only remove information with respect to a specific class of classifiers. It could always be the case that complex correlations between the inputs and the guarded attributes exist, and that an adversary can try to exploit them to predict the guarded attribute if this class of classifiers is not too complex. Our use of kernels alleviates some of this issue, though not completely.

Finally, experimentally, we focus on text only in English. It is not clear to what extent our method generalizes to other languages in a useful manner, especially when morphology is rich, and the neural representations encode important information for the task at hand, but that information would be removed by our method.

\section*{Ethical Considerations}
Public trust plays a significant role in the broad applicability of NLP in real-world scenarios, especially in critical situations that may directly impact people's lives. NLP research of the kind presented in this paper helps this issue take the spotlight it deserves. However, we discourage NLP practitioners from using our method (and similar methods) as an out-of-the-shelf solution in deployed systems. We recommend investing a significant amount of time and effort in understanding the applicability and universality of our method to the debiasing of representations. Issues such as expected type of adversariality or tolerance level for drop in system performance need to be considered.


\section*{Acknowledgments}

We thank the anonymous reviewers for their helpful comments.
We especially appreciate the comment one of the reviewers provided regarding our title. Particularly, it could be misinterpreted as an indication of frustration at rejections of our paper (``gold'') in previous conferences. Rather, the ``gold'' in our case is the low-intensity principal vectors, which are pruned in most use cases of SVD.
We also thank Shauli Ravfogel for providing support with the code for INLP, Ryan Cotterell for discussions and Matt Grenander for feedback on earlier drafts.
The experiments in this paper were supported by compute grants from the Edinburgh Parallel Computing Center (Cirrus) and from the Baskerville Tier 2 HPC service (University of Birmingham).


\bibliography{anthology,custom}
\bibliographystyle{acl_natbib}

\appendix

\ignore{
\paragraph{How is this related to Latent Semantic Analysis?} In LSA, we perform SVD on a word-document co-occurrence matrix, aiming to project either one type of vectors into a space that corresponds to the largest singular vectors.
}

\section{Eigenvectors of $\mat{\Lambda}$}
\label{appendix:A}

We turn to the following Lemma used in \S\ref{section:kernel}.
\begin{lem}
Let $\myvec{w}$ be an eigenvector associated with eigenvalue $\lambda \in \mathbb{R}$ for $\mat{\Gamma} = \mat{K}_{\phi} \mat{K}_{\psi}$. Then $\mat{\Phi} \myvec{w}$ is an eigenvector of $\mat{\Omega}\mat{\Omega}^{\top}$.
\end{lem}

\begin{proof}

Since $\myvec{w}$ is an eigenvector of $\mat{\Gamma}$, it holds that $\Gamma \myvec{w} = \lambda \myvec{w}$. Therefore:

\begin{align}
\mat{\Psi}^{\top}\mat{\Psi}\mat{\Phi}^{\top}\mat{\Phi}
\myvec{w} & = \lambda \myvec{w}, \\
\left( \mat{\Phi} \mat{\Psi}^{\top}\mat{\Psi}\mat{\Phi}^{\top} \right) \mat{\Phi} \myvec{w} & = \lambda \mat{\Phi} \myvec{w},
\end{align}

and therefore $\mat{\Phi}\myvec{w}$ is an eigenvalue of 

\begin{equation}
    \mat{\Omega}\mat{\Omega}^{\top} = \mat{\Phi}\mat{\Psi}^{\top}\mat{\Psi}\mat{\Phi}^{\top}.
\end{equation}

\end{proof}

\section{Nearest Neighbors Test for Word Embedding Debiasing}
\label{sec:knn}
\label{appendix-b}

We give in Table~\ref{tab:KNN_GloVe} the ten nearest neighbor words for ten random words from the data, before and after using SAL. The neighboring words are determined through cosine similarity of the corresponding embeddings with respect to the pivot word embedding. We observe little to no difference in these two lists (before and after the removal).

\begin{table*} [hbt]

\centering
{
\begin{tabular}{|l|l|l|l|l|l|l|l|l|l|l cc}
\hline
\multicolumn{1}{|c|}{Words} &
\multicolumn{1}{|c|}{Nearest neighbors (before)}&
\multicolumn{1}{|c|}{Nearest neighbors (after)}\\

\hline
lobbying & lobbyists, lobbyist, campaigning &  lobbyists, lobbyist, campaigning\\
once  & again, then, when & again, then, when \\
parliament & parliamentary, mps, elections & parliamentary, mps, elections \\
dashboard  & dashboards, smf, powered & dashboards, smf, powered \\
cumulative  & gpa, accumulative, aggregate & gpa, accumulative, aggregate \\
foam  & rubber, mattress, polyurethane & rubber, mattress, polyurethane \\
rh  & lh, bl, r & lh, bl, graphite \\
genetically & gmo, gmos, genetic & gmo, gmos, genetic \\
inner & outer, inside, innermost & outer, inside, innermost \\
harvest & harvesting, harvests, harvested & harvesting, harvests, harvested \\
secretary & deputy, minister, treasurer & deputy, minister, secretaries\\
\hline
ruth & helen, esther, margaret & helen, esther margaret \\
charlotte & raleigh, nc, atlanta & raleigh, nc, atlanta \\
abigail & hannah, lydia, eliza & hannah, lydia, samuel \\
sophie & julia, marie, lucy & julia, lucy, claire \\
nichole & nicole, kimberly, kayla & nicole, kimberly, mya \\
emma & emily, lucy, sarah & emily, watson, sarah \\
\hline
david & stephen, richard, michael & alan, stephen, richard \\
richard & robert, william, david & robert, william, david \\
joseph & francis, charles, thomas & mary, francis, charles \\
thomas & james, william, john & james, william, henry \\
james & john, william, thomas & william, john, thomas \\
\hline
\end{tabular}
}
\caption{Nearest neighbor test on GloVe word embeddings before and after debiasing on gender. The upper block includes a random set of words, while the middle and bottom block include female and male names.}
\label{tab:KNN_GloVe}
\end{table*}

\end{document}